\documentclass{article}

\usepackage[pagebackref=true]{hyperref}
\usepackage[margin=1.2in]{geometry}
\usepackage{authblk}
\usepackage[sort,numbers]{natbib}
\date{}

\usepackage[utf8]{inputenc} 
\usepackage[T1]{fontenc}    
\usepackage{hyperref}       
\usepackage{url}            
\usepackage{booktabs}       
\usepackage{amsfonts}       
\usepackage{nicefrac}       
\usepackage{microtype}      
\usepackage{amssymb,amsmath,amsthm}
\usepackage{algorithm}
\usepackage{algpseudocode}
\usepackage{enumitem}

\newtheorem{theorem}{Theorem}

\newtheorem{lemma}{Lemma}

\newtheorem{definition}{Definition}

\newtheorem{assumption}[definition]{Assumption}

\DeclareMathOperator*{\argmin}{arg\,min} 
\DeclareMathOperator{\real}{\mathbb{R}}

\DeclareMathOperator{\dom}{dom}

\newcommand\norm[1]{\left\lVert#1\right\rVert}
\newcommand\abs[1]{\left\lvert#1\right\rvert}

\newcommand\encase[1]{\left[ #1 \right]}
\newcommand\encaser[1]{\left( #1 \right)}
\newcommand\encasecurly[1]{\left\{ #1 \right\}}

\DeclareMathOperator{\expect}{\mathbb{E}}

\newcommand \dotprod[2]{{\langle {#1},{#2} \rangle}}

\newcommand{\defeq}{\mathrel{\overset{\makebox[0pt]{\mbox{\normalfont\tiny\sffamily def}}}{=}}}

\newcommand{\dv}{\mathbf{d}}

\newcommand{\sv}{\mathbf{s}}

\newcommand{\uv}{\mathbf{u}}
\newcommand{\vv}{\mathbf{v}}
\newcommand{\wv}{\mathbf{w}}
\newcommand{\xv}{\mathbf{x}}
\newcommand{\yv}{\mathbf{y}}
\newcommand{\zv}{\mathbf{z}}

\newcommand{\0}{ {\bf 0}}

\renewcommand{\dom}{\mathcal{Q}}
\newcommand{\good}{{\texttt{successful}}}
\newcommand{\bad}{{\texttt{unsuccessful}}}
\newcommand{\verygood}{{\texttt{very-successful}}}

\title{Global linear convergence of Newton's method without strong-convexity or Lipschitz gradients}

\author{
  Sai Praneeth Karimireddy, Sebastian U. Stich, Martin Jaggi\\\vspace*{-0.3cm}
  EPFL\\\vspace*{-0.1cm}
  	\texttt{\small \{sai.karimireddy,~sebastian.stich,~martin.jaggi\}@epfl.ch} 
}

\begin{document}

\maketitle
\vspace*{-0.5cm}
\begin{abstract}
  We show that Newton's method converges globally at a linear rate for objective functions whose Hessians are \emph{stable}. This class of problems includes many functions which are not strongly convex, such as logistic regression.
Our linear convergence result is (i) affine-invariant, and holds even if an (ii) approximate Hessian is used, and if the subproblems are (iii) only solved approximately. Thus we theoretically demonstrate the superiority of Newton's method over first-order methods, which would only achieve a sublinear $O(1/t^2)$ rate under similar conditions.
\end{abstract}


\section{Introduction}\label{sec:intro}
Newton's method is one of the earliest algorithms for the minimization of an unconstrained convex objective function $f \colon \real^n  \rightarrow \real$, \vspace{-1mm}
\begin{equation}\label{eqn:objectiveSimple}
\min_{\xv \in \real^n} \ f(\xv)\,,
\end{equation}
and iteratively performs the following update for some step-size $\gamma > 0$,
\begin{equation}\label{eqn:newton-update}
  \xv_{t+1} \leftarrow \xv_t - \gamma [\nabla^2 f(\xv_t)]^{-1}\nabla f(\xv_t)\,.
\end{equation}
Here $f$ is assumed to be a twice differentiable convex function.
In contrast to the classical literature 
we do not assume that the function $f$ is \emph{smooth} (i.e. that the gradient is Lipschitz continuous), nor do we assume \emph{strong convexity}.

Popularized in its present form by Bennet \cite{bennett1916newton} and Kantarovich \cite{kantorovich1948functional}, Newton's method has been an immensely important algorithm for optimization. Though there has been significant work analyzing and extending the standard scheme \eqref{eqn:newton-update}, global convergence results remain few and unsatisfactory (cf.~\cite{nesterov_cubic_2006} and references therein). In a seminal result, Nesterov and Nemirovski \cite{NesterovNemi1994} show that
Newton's algorithm achieves local quadratic convergence.
However, the conditions under which quadratic convergence occurs are too restrictive---they require both the function to be self-concordant, and the starting point to be almost at the optimum. Neither of these conditions is typically satisfied when applying Newton's method for minimizing functions of the form \eqref{eqn:objectiveSimple} in applications.
Most of the global convergence results are either i) hard to compare with gradient descent and make strong assumptions on $f(\xv)$ (e.g. \cite{NesterovNemi1994,polyak_newton-kantorovich_2006}), or ii) have a rate which is slower than vanilla gradient descent (e.g. \cite{gurbuzbalaban2015globally,lee_inexact_2018}). An exception to this is the breakthrough result by Nesterov and Polyak~\cite{nesterov_cubic_2006} where they obtain a $O(1/t^2)$ rate, and later $O(1/t^3)$ (\cite{nesterov_accelerating_2005}), by solving cubic sub-problems. These rates do not assume strong convexity or Lipschitz gradients.
However, solving cubic sub-problems is impractical even for medium sized problems.

On the other hand, there has been recent efforts in performing efficient approximations of \eqref{eqn:newton-update} in time comparable to that required for a gradient update~(\cite{agarwal_second-order_2016,gower_accelerated_2018,lee_inexact_2018}). These methods, so far, did not enjoy any global convergence rates better than first-order methods.

\paragraph{A new regularity condition.}~Most analyses of Newton-type algorithms assume that for $\xv \approx \yv$, the Hessians are also close $\nabla^2 f(\xv) \approx \nabla^2 f(\yv)$. In particular the results on cubic regularization (e.g. \cite{nesterov_cubic_2006}) assume that the Hessian is Lipschitz. This is equivalent to assuming that the condition $\nabla^2 f(\xv) \approx \nabla^2 f(\yv)$ holds with an \emph{additive} error whose magnitude depends on the distance $\norm{\xv -\yv}_2$. We instead assume that the Hessian is \emph{stable} which means that the error is \emph{multiplicative}. This is sufficient to show a simple proof of the global linear convergence of Newton's method. Further, since our condition is multiplicative, stability is also a scale-free (i.e. affine invariant) condition.

The assumption of a stable Hessian was previously used to analyze the statistical properties of logistic regression in \cite{bach_self-concordant_2010}, and to analyze the convergence of SGD on logistic regression in \cite{bach_non-strongly-convex_2013,bach_adaptivity_2014}. We were inspired by~\cite{cohen_matrix_2017} who obtain an efficient algorithm for matrix scaling using ideas very similar to here.

\paragraph{Our contributions.}
Our main contribution is a straightforward affine-invariant proof for global linear convergence of Newton's algorithm, without resorting to strong convexity or Lipschitz continuous gradient (Section \ref{sec:convergence}). We instead rely only on a natural multiplicative notion of stability of the Hessian (Section \ref{sec:stability}. This shows an exponential gap between global convergence rates of first-order and second-order methods for a wide class of functions, placing Newton-type methods on a strong theoretical footing.
Further, in Section \ref{sec:trust}, we relax stability and show that a \emph{local} notion of stability is sufficient to guarantee linear convergence for trust-region Newton methods. Finally, we show in Section \ref{sec:extensions} that linear convergence persists when using inexact and proximal Newton steps.

\paragraph{Related work.}
Newton's method with backtracking has been shown to be globally convergent for self-concordant functions (\cite{NesterovNemi1994}) but the resulting rate is difficult to compare directly to gradient-based methods due to its two-phase additive structure.
Otherwise, global convergence results of second-order algorithms were known when $f(\xv)$ has both strong-convexity and Lipschitz gradients~(\cite{polyak_newton-kantorovich_2006,lee_inexact_2018}), or by solving cubic subproblems (\cite{nesterov_cubic_2006,nesterov_accelerating_2005,CartisAdaptivecubicregularisation2011a}). Similar convergence rates are shown for the inexact Newton method in (\cite{scheinberg2016practical,lee_inexact_2018}). Empirically, (\cite{lin_trust_2008}) show that trust region Newton's method significantly outperform other methods, and is hence the default optimization algorithm for a variety of problems in the widely used LIBLINEAR library (\cite{fan2008liblinear}). Although in this work we restrict ourselves to convex functions, Newton-type algorithms~(\cite{nesterov_cubic_2006,agarwal2017finding,royer2018complexity}) as well as trust region methods (\cite{ConnTrustregionMethods2000,curtis2017trust}) have been successfully used to escape saddle points and converge to a local minimum in non-convex settings.

\section{Stability of the Hessian}\label{sec:stability}
We now formally define our notion of a stable Hessian and show that it is implied by many other standard assumptions.
We will also demonstrate that for a large class of problems on which Newton's method is usually applied, our condition is satisfied.
As is standard, we will assume that the level set of the function $F(\xv)$ is bounded. In particular set $\dom$ has a bounded diameter $D$ where $\dom$ is defined as
\begin{equation}\label{eqn:domain}
	\dom = \{\xv \,|\, F(\xv) \leq F(\xv_0)\}\,.
\end{equation}

\subsection{Definition of stability}
Here we present an affine invariant definition of a stable Hessian. For any vector $\vv \in \real^n$, and a positive semi-definite matrix $M \in \real^{n\times n}$, let $\norm{\vv}_M^2$ denote the semi-norm $\vv^\top M \vv$.
\begin{assumption}[$c$-stable Hessian]\label{asm:stable-hess-formal}
	For any $\uv, \vv \in \dom$ and $\uv \neq \vv$, we assume $\norm{\vv - \uv}_{\nabla^2 f(\uv)} > 0$ and that there exists a constant $c \geq 1$ such
that\footnote{This assumption can be relaxed---instead of for all of $\dom$, we only need the condition to for hold for $\uv = \xv_t$, and $\vv =(1-\alpha)\xv_t + \alpha\xv_{t+1}$ as well as $\vv = (1-\alpha)\xv_t + \alpha\xv^\star$, for all $t \geq 0$ and $\alpha \in (0,1]$.}
	\[
		c \defeq \max_{\uv, \vv \in \dom} \frac{\norm{\vv - \uv}_{\nabla^2 f(\vv)}^2}{\norm{\vv - \uv}_{\nabla^2 f(\uv)}^2} \,.
	\]
\end{assumption}
Assumption \ref{asm:stable-hess-formal} 
allows to derive \emph{global} upper and lower bounds on the function $f(\xv)$ for $\xv \in \dom$.
In contrast to standard assumptions such as strong convexity, smoothness or Lipschitz Hessian, stability is affine invariant:
\begin{lemma}\label{lem:affine-inv}
	The constant $c$ defined in Assumption \ref{asm:stable-hess-formal} is invariant under any non-singular linear transformations of $f(\xv)$.
\end{lemma}
\subsection{Sufficient conditions}
Here we will discuss a host of standard assumptions and see how they imply a stable Hessian.
The formal definitions of the conditions, as well as the proof of Theorem \ref{thm:sufficient} is presented in Appendix \ref{sec:sufficient}.
We also assume that the domain $\dom$ is bounded with a diameter $D = \max_{\uv,\vv \in \dom}\norm{\uv - \vv}_2$.
\begin{theorem}\label{thm:sufficient}
  The following are sufficient conditions for ensuring the stability of the Hessian as defined in Assumption \ref{asm:stable-hess-formal}:
  \begin{enumerate}[label=(\roman*),itemsep=-.5ex,topsep=-.5ex]
    \item $L$-Lipschitz gradient and $\mu$-strongly convex $\,\Rightarrow\,$ $(L/\mu)$-stable Hessian,
    \item $M$-Lipschitz Hessian and $\mu$-strongly convex $\,\Rightarrow\,$ $\bigl( 1 + \frac{MD}{\mu}\bigr)$-stable Hessian,
    \item $k$-self-concordant and $L$-Lipschitz gradient $\,\Rightarrow\,$ $(1 + kDL)^2$-stable Hessian, and
    \item $k$-quasi-self-concordant $\,\Rightarrow\,$ $\exp(kD)$-stable Hessian.
  \end{enumerate}
\end{theorem}
\subsection{Applications}\label{subsec:applications}
For a given matrix $A$ we consider functions of the form $f(A\xv)$ where $f$ is coordinate-wise separable. For learning applications $A$ is typically the data matrix. The objective function may further be regularized for an arbitrary $g(\xv)$ (e.g. $L_1$ regularizer), as we will discuss in Section \ref{sec:extensions}. We can assume that $A$ is full-rank, otherwise one can restrict the domain $\dom$ to the range of $A$. Further let us also assume that each row $A_i$ of the matrix $A$ is normalized and $\norm{A_i}_\star =1$. Then the affine-invariance of stability allows to transform $f(A\xv) = f(\uv)$ into a sum $\sum_i f_i(u_i)$ of $n$ one-dimensional functions where $u_i = A_i^\top\xv$. Since $A_i$ is normalized, $|\Delta u_i|\leq \norm{A_i}_\star\norm{\Delta \xv} \leq \norm{\Delta \xv} \leq D$.
Thus without loss of generality, we can focus on discussing the stability of one-dimensional functions with a domain diameter less than $D$. Many of the following applications have been adapted from \cite{SunGeneralizedSelfConcordantFunctions2017}.
\begin{enumerate}[label=(\alph*),itemsep=-.5ex,topsep=-.8ex]
	\item{\em Logistic regression:} The loss function $f(x) = \log(1 + e^{-x})$ is shown to be $1$-quasi-self concordant in \cite{bach_self-concordant_2010}, and so is $\exp(D)$-stable.
	\item{\em Wasserstein distance:} Functions of the form $e^x - x$ are also $\exp(D)$-stable. The dual of the entropy-regularized Wasserstein distance is of this form \cite{cuturi_sinkhorn_2013}.
	\item{\em Boosting:} Ada-boost can be seen as a first-order algorithm on an exponential loss function (cf. Chapter 6,~\cite{shalev2007online}).
	\item{\em Self-concordant functions:} As was shown in Theorem \ref{thm:sufficient}, all self-concordant functions (e.g. logarithmic barriers) with bounded domain and Lipschitz gradients are stable.
	\item{\em Entropy regularizer:} The standard entropy function $f(x) = x\ln x$ also fits into our framework, assuming bounded domain $x \in [a,b]$ for $ a >0$. The Hessian of the entropy function is $f''(x) = 1/x$ and so is $\frac{b}{a}$-stable.
	\item{\em Robust regression:} Instead of the standard least-squares loss, \cite{xu2009robust} consider a more robust version which is $f(x) = x^q$ for $q \in (1,2]$ with a Hessian $f''(x) = q(q-1)x^{q-2}$. Assuming a bounded domain $x \in [a,b]$ for $ a >0$, the function is $((b/a)^{2-q})$-stable.
\end{enumerate}
While some of these constants may seem large (e.g. the $\exp(D)$ in Logistic regression), in Section \ref{sec:trust} we will see a \emph{local} notion of stability which gets around the super-linear dependence on $D$.
\section{Convergence of exact Newton's method}\label{sec:convergence}
The convergence of Newton's method follows in a straightforward manner from the definition of a stable Hessian. To demonstrate the core idea, let us look at the simplest case---Newton's algorithm on a twice differentiable function $f(\xv)$ using the exact inversion of the Hessian (or its pseudo-inverse), as presented in Algorithm \ref{alg:exact-newton}. We will later extend the algorithm and relax many of these assumptions. The algorithm uses a fixed step-size $1/\sigma$. This can easily be made adaptive (see Appendix \ref{sec:line-search}) at a mild additional cost.
\begin{algorithm}
	\caption{Exact Newton Descent}
		\label{alg:exact-newton}
	\begin{algorithmic}[1]
		\State \textbf{Input:} $\xv_0$ and $\sigma$. 
		\For{$t = \{0,\dots\}$}
			\State $\xv_{t+1} \leftarrow \xv_t - \frac{1}{\sigma}[\nabla^2 f(\xv_t)]^{\dagger}\nabla f(\xv_t)$
		\EndFor
	\end{algorithmic}
\end{algorithm}

\begin{theorem}\label{thm:exact-newton}
 Given Assumption \ref{asm:stable-hess-formal}, for any iteration $T \geq 0$ of Algorithm \ref{alg:exact-newton} with $\sigma \geq c$,
 \[
 	f(\xv_T) - f(\xv^\star) \leq \Big(1 - \frac{1}{c\sigma}\Big)^T [ f(\xv_0) - f(\xv^\star) ]\,.
 \]
\end{theorem}

As we noted before, the assumption that the Hessian is stable allows to provide global upper and lower bounds on the function value (the proof is given in Appendix \ref{subsec:simple-bounds-proof}).
\begin{lemma}\label{lem:upperlower-simple}
	Given Assumption \ref{asm:stable-hess-formal}, for any $\xv, \yv \in\dom$,
	\begin{align}
		\text{Upper bound:}\quad f(\yv) &\leq f(\xv) + \dotprod{\nabla f(\xv)}{\yv - \xv} + \frac{c}{2}\norm{\yv - \xv}_{\nabla^2 f(\xv)}^2  \,\text{, and}\label{eqn:upper-simple}\\
		\text{Lower bound:}\quad f(\yv) &\geq f(\xv) + \dotprod{\nabla f(\xv)}{\yv - \xv} + \frac{1}{2c}\norm{\yv - \xv}_{\nabla^2 f(\xv)}^2  \,.\label{eqn:lower-simple}
	\end{align}
\end{lemma}
The bounds above only hold for $\xv \in \dom$ as defined in \eqref{eqn:domain}. To use the Lemma, we need that $\xv_t \in \dom$ for all $t \geq 0$. For this, it suffices to show that Algorithm \ref{alg:exact-newton} is a descent method. For now, let us assume this technicality---the proof can be found in Appendix~\ref{sec:proofdescent}.
\begin{lemma}\label{lem:descent}
  Under Assumption \ref{asm:stable-hess-formal}, for any $t \geq 0$ of Algorithm \ref{alg:exact-newton} with $\sigma \geq c$, the update $\frac1\sigma [\nabla^2 f(\xv_t)]^\dagger \nabla f(\xv_t)$ is well-defined and further $f(\xv_{t+1}) \leq f(\xv_t)$.
\end{lemma}
\textbf{Proof of Theorem \ref{thm:exact-newton}.}
By Lemma \ref{lem:descent}, Algorithm \ref{alg:exact-newton} is well-defined and is a descent method.
This means that both $\xv_t$ and $\xv_{t+1}$ lie in $\dom$ and we can apply Lemma \ref{lem:upperlower-simple}. The upper bound \eqref{eqn:upper-simple} implies that for $\sigma \geq c$,
\begin{align*}
	f(\xv_{t+1}) &\leq f(\xv_t) + \dotprod{\nabla f(\xv_t)}{\Delta \xv_t} + \frac{\sigma}{2}\norm{\Delta \xv_t}_{\nabla^2 f(\xv_t)}^2\\
	&= f(\xv_t) - \frac{1}{2 \sigma}\norm{\Delta \xv_t}_{\nabla^2 f(\xv_t)}^2\,.
\end{align*}
Here note that $\Delta \xv_t = [\nabla^2 f(\xv_t)]^{\dagger}\nabla f(\xv_t)$.
Now minimizing both sides of the lower bound \eqref{eqn:lower-simple} gives
\[
	f(\xv^\star) \geq f(\xv) - \frac{c}{2}\norm{[\nabla^2 f(\xv)]^{\dagger}\nabla f(\xv)}_{\nabla^2 f(\xv)}^2\,.
	\]
	Using the above bound with $\xv = \xv_t$, we get
	\begin{align*}
	f(\xv_{t+1}) &\leq f(\xv_t) - \frac{1}{2 \sigma}\norm{\Delta \xv_t}_{\nabla^2 f(\xv_t)}^2\\
	&\leq f(\xv_t) + \frac{1}{c\sigma}[f(\xv^\star) - f(\xv_t)]
	\end{align*}
	Subtracting $f(\xv^\star)$ from both sides, and iterating from $0$ to $T$ proves the theorem.
\qed

\section{Trust region Newton's method}\label{sec:trust}
The convergence rate of Newton's algorithm in Theorem \ref{thm:exact-newton} critically depends on the constant $c$, which is a \emph{global} measure bounding the relative change of the Hessian of $f(\xv)$ around the current point~$\xv$. Often, the value of $c$ depends on the diameter $D$ of the domain $\dom$. E.g., as we discussed in Section~\ref{subsec:applications}, Logistic regression and exponential loss are $\exp(D)$-stable, which can be a large value. In this section, we design an algorithm whose convergence depends only on a \emph{local} measure of stability, getting around the potentially exponential dependence on $D$.

\subsection{Local stability}\label{subsec:local-stability}
We introduce a \emph{local} measure of stability $d$, which is typically much smaller than $c$. This notion captures the multiplicative change in the Hessian in a small ball of radius $r$ around the current point $\xv$, measured in an arbitrary norm $\norm{\cdot}$.
\begin{assumption}[$d(r)$-locally stable with respect to $\norm{\cdot}$]\label{asm:local-stable-hess-formal}
	For any $\uv, \vv \in\dom$ such that $\uv \neq \vv$ and $\norm{\uv - \vv}\leq r$, we assume that $\norm{\vv - \uv}_{\nabla^2 f(\uv)} > 0$ and that there exists a constant $d(r) \geq 1$ for which the following holds
	\[
	d(r) \defeq \max_{\norm{\uv - \vv}\leq r} \frac{\norm{\vv - \uv}_{\nabla^2 f(\vv)}^2}{\norm{\vv - \uv}_{\nabla^2 f(\uv)}^2} \,.
	\]
\end{assumption}
Since the norm $\norm{\cdot}$ may not be affine invariant, the resulting constant $d(r)$ is also not necessarily affine invariant. It is, however, possible to circumvent this limitation (refer Section \ref{sec:appendix-affine-trust} in the Appendix).

\subsection{Trust-region Algorithm}
Trust-region methods restrict each update to a small ball of radius $r$ around $\xv$, and so are more `local' algorithms.%
%
\begin{algorithm}
	\caption{Trust-region Newton Descent}
		\label{alg:trust-newton}
	\begin{algorithmic}[1]
		\State \textbf{Input:} $\xv_0$, $r > 0$, and $\sigma$.
		\For{$t = \{0,\dots\}$}
        \State $\xv_{t+1} \leftarrow \argmin_{\norm{\yv - \xv_t} \leq r} \dotprod{\nabla f(\xv_t)}{\yv - \xv_t} + \frac{\sigma}{2}\norm{\yv - \xv_t}^2_{\nabla^2 f(\xv_t)}$\label{eqn:trust-update}
		\EndFor
	\end{algorithmic}
\end{algorithm}%
\subsection{Convergence analysis}
\begin{theorem}\label{thm:trust-newton}
	Given Assumption \ref{asm:local-stable-hess-formal}, for any iteration $T \geq 0$ of Algorithm \ref{alg:trust-newton} with $\sigma \geq d(r)$,
	\[
		f(\xv_{T}) - f(\xv^\star) \leq \Big(1 - \frac{r}{D\sigma d(r)}\Big)^T \big[f(\xv_{0}) - f(\xv^\star)\big]\,,
	\]
	where $r$ is the trust region radius and $D$ is the diameter of the level set i.e. $D = \max_{\xv,\yv \in \dom} \norm{\xv - \yv}$\,.
\end{theorem}
\begin{proof}
The proof of Theorem \ref{thm:trust-newton} is very similar to that of Theorem \ref{thm:exact-newton}. The main deviation is the derivation of tighter lower and upper bounds that depend on the local bound $d(r)$ instead of the global parameter $c$.
This is detailed in Lemma~\ref{lem:upper-lower-bound} in Appendix~\ref{sec:proof-of-trustregion}. At any iteration $t \geq 0$, we get that for any $\yv$ such that $\norm{\yv - \xv_t} \leq r$ the following holds
\begin{align}
	\text{Upper bound:} \quad F(\xv_{t+1}) - F(\xv_t) &\leq \dotprod{\nabla f(\xv_t)}{\yv - \xv_t} + \frac{d(r)}{2}\norm{\yv -\xv_t}^2_{\nabla^2 f(\xv_t)} \label{eqn:trust-upper-sigma}\,,\\
  \text{Lower bound:} \quad F(\xv_{t+1}) - F(\xv_t) &\geq \dotprod{\nabla f(\xv_t)}{\yv - \xv_t} + \frac{1}{2 d(r)}\norm{\yv -\xv_t}^2_{\nabla^2 f(\xv_t)} \label{eqn:trust-lower-sigma}\,.
\end{align}
The upper bound \eqref{eqn:trust-upper-sigma} combined with the update in Step \ref{eqn:trust-update} implies that for any $\sigma \geq d(r)$,
\begin{align*}
  f(\xv_{t+1}) - f(\xv_t) &\leq \dotprod{\nabla f(\xv_t)}{\xv_{t+1} - \xv_t} + \frac{\sigma}{2}\norm{\xv_{t+1} -\xv_t}^2_{\nabla^2 f(\xv_t)}\\
  &= \min_{\yv,\,\norm{\yv - \xv_t}\leq r}\dotprod{\nabla f(\xv_t)}{\yv - \xv_t} + \frac{\sigma}{2}\norm{\yv -\xv_t}^2_{\nabla^2 f(\xv_t)}\\
  &\leq \frac{1}{\sigma d(r)}\ \min_{\yv,\,\norm{\yv - \xv_t}\leq r}\dotprod{\nabla f(\xv_t)}{\yv - \xv_t} + \frac{1}{2 d(r)}\norm{\yv -\xv_t}^2_{\nabla^2 f(\xv_t)}
\end{align*}
The last inequality is trivial (with an equality) when minimizing unbounded quadratics, but is also valid when minimizing over convex domains (refer Lemma \ref{lem:changing-sigma} in Appendix \ref{sec:proof-of-trustregion}). Let us define $\gamma = r/D$ and the point $\xv^\star_\gamma = (1 - \gamma)\xv_t + \gamma\xv^\star$. Then \[\norm{\xv^\star_\gamma - \xv_t} = \norm{\gamma(\xv^\star - \xv_t)} = \gamma\norm{\xv^\star - \xv_t} \leq r\,.\]
Combining this with our previous observation gives
\begin{align*}
  f(\xv_{t+1}) - f(\xv_t) &\leq \frac{1}{\sigma d(r)} \ \min_{\yv,\,\norm{\yv - \xv_t}\leq r}\dotprod{\nabla f(\xv_t)}{\yv - \xv_t} + \frac{\sigma}{2}\norm{\yv -\xv_t}^2_{\nabla^2 f(\xv_t)}\\
  &\leq \frac{1}{\sigma d(r)} \dotprod{\nabla f(\xv_t)}{\xv^\star_\gamma - \xv_t} + \frac{\sigma}{2}\norm{\xv^\star_\gamma -\xv_t}^2_{\nabla^2 f(\xv_t)}\\
  &\leq \frac{1}{\sigma d(r)} [f(\xv^\star_\gamma) - f(\xv_t)]\,.
\end{align*}
The last inequality used the lower bound from \eqref{eqn:trust-lower-sigma}. Now we will have to relate the term $f(\xv^\star_\gamma)$ to the actual minimum value $f(\xv^\star)$. This we will do by using the convexity of the function $f(\xv)$.
\begin{align*}
  f(\xv_{t+1}) - f(\xv_t) &\leq \frac{1}{\sigma d(r)} [f(\xv^\star_\gamma) - f(\xv_t)]\\
  &= \frac{1}{\sigma d(r)} [f((1 - \gamma)\xv_t + \gamma\xv^\star) - f(\xv_t)]\\
  &\leq \frac{1}{\sigma d(r)} [(1 - \gamma)f(\xv_t) + \gamma f(\xv^\star) - f(\xv_t)]\\
  &= \frac{\gamma}{\sigma d(r)}[f(\xv^\star) - f(\xv_t)]\,.
\end{align*}
Adding and subtracting $f(\xv^\star)$ from the left side, rearranging the terms, and iterating over $t$ finishes the proof.
\end{proof}
\subsection{Improvement in the rate of convergence}\label{subsec:trust-rate-improved}
In a number of applications we saw in Section \ref{subsec:applications}, the dependence of $c$ on the diameter $D$ was super-linear (and even exponential). Local-stability gets around this and ensures that the rate of convergence of the Algorithm \ref{alg:trust-newton} depends at most linearly on $D$.

For $\sigma = c$ in Theorem \ref{thm:exact-newton} gives a rate depending on $c^2$. In contrast, using $\sigma = d(r)$, Theorem \ref{thm:trust-newton} gives a rate depending on $d^2(r)/r$. Thus the optimal $r$ can be computed as
\[
  r^\star = \argmin_{r}d^2(r)/r\,.
\]
As an illustrative example, consider logistic regression or exponential losses. The local-stability scales as $d(r) = \exp(r)$ for $r \in [0,D]$. The rate of convergence of Newton's method would depend on $c^2 = e^{2D}$. On the other hand, using the optimal trust region radius $r^\star = 1/2D$, the rate for the trust-region method becomes $2eD$.
This result makes a very strong case for using trust-region Newton methods.

There are two points to note here. First, one might ask if a similar improvement could be shown for the simpler Newton search equipped with a line search. We answer in the negative in Section \ref{sec:optimality}. Next, as we noted before, trust region methods are not affine-invariant, and moreover require solving the Newton step with an additional constraint. In the appendix (Section \ref{sec:appendix-affine-trust}), we show an affine-invariant algorithm only requiring minimizing quadratics over the domain $\dom$.
\section{Approximate and proximal extensions}\label{sec:extensions}
We can extend our analysis of Newton's method to the proximal setting to minimize a composite objective function, i.e.\vspace{-2mm}
\begin{equation}\label{eqn:objective}
F(\xv^\star) \defeq \min_{\xv \in \real^n}\encasecurly{F(\xv) \defeq f(\xv) + g(\xv)}\,,
\end{equation}
where $f: \real^n \rightarrow \real$ as before is a twice differentiable convex function, and $g: \real^n \rightarrow \real \cup \{+\infty\}$ is a possibly non-differentiable, extended valued convex function.
\subsection{Inexact Newton steps}
In this section we also make two relaxations, one being that an exact Hessian is used, and second that the quadratic subproblem is solved exactly.
At each iteration $t$ with iterate $\xv_t$, we assume access to the exact gradient $\nabla f(\xv_t)$, and only an approximation $H_t\in R^{n\times n}$ of the Hessian $\nabla^2 f(\xv_t)$.
\paragraph{Approximate Hessian.}
Below we list a few scenarios where this notion of an approximate Hessian is useful:
\begin{enumerate}[itemsep=-.5ex,topsep=-.8ex]
	\item{\em Sketched Hessian.} In machine learning and signal processing applications, the function $f(\xv)$ is typically of the form $l(A\xv)$ where $l$ is a simple, separable function and $A$ is a data matrix. In such cases, the Hessian $\nabla^2 f(\xv) = A^\top \nabla^2 l(A\xv) A$ where $\nabla^2 l(A\xv)$ is very cheap to compute (same cost as computing the gradient). Instead of using the full matrix $A$, a low dimensional sketch $S_t A$ is used instead. This provides guarantees satisfying \eqref{asm:hessian-quality} while ensuring cheap update steps~(cf. \cite{gower_randomized_2016,gower_accelerated_2018}).
	\item{\em Hessian free inexact methods.} If we use first order algorithms to minimize $Q_t^\sigma$, we would only require products of the Hessian with a vector. Such  product can be computed without computing, or storing the entire Hessian matrix. The resulting algorithms are inexpensive and costs are comparable to first order methods (cf.~\cite[Section 6.1]{bottou_optimization_2018}).
	\item{\em Block diagonal $H_t$.} For distributed and parallel computation, it is crucial that we are able to create subproblems such that they are \emph{separable} i.e. we can decompose the subproblem into multiple subproblems which can be solved independently~(for e.g. \cite{smith_cocoa:_2016-1,karimireddy_adaptive_2018-1,gargiani_hessian-cocoa:_2017}).
\end{enumerate}
\paragraph{Approximate subproblems.}
Using $H_t$, we form a subproblem $Q_t^\sigma(\Delta \xv)$ as in Step \ref{eqn:subproblem}.
 Then we assume that at each iteration, our subproblem is solved to an arbitrary multiplicative accuracy, and only in expectation over some randomness of the subproblem algorithm. In particular, we assume the update is computed as in Step \ref{eqn:accuracy} for any fixed $\Theta \in (0,1]$. Note that if $\Theta = 1$, this means  $Q_t^\sigma(\Delta \xv_t) = \min_{\Delta\xv,\,\norm{\Delta\xv} \leq r}Q_t^\sigma(\Delta\xv)$ and that the subproblem was solved exactly.
\begin{algorithm}
	\caption{Approximate and Proximal Newton Descent}
		\label{alg:approx-newton}
	\begin{algorithmic}[1]
		\State \textbf{Input:} $\xv_0$ and $\sigma$.
		\For{$t = \{0,\dots\}$}
		    \State \emph{Define subproblem:}
		    $Q_t^\sigma(\Delta \xv) \defeq \dotprod{\nabla f(\xv_t)}{\Delta \xv} + \frac{\sigma}{2}\norm{\Delta \xv}_{H_t}^2 + g(\xv_t + \Delta \xv) - g(\xv_t)$ \label{eqn:subproblem}
			\State \emph{Approximately minimize subproblem:} Find $\Delta \xv_t$ such that
			\State \hspace{0.3cm} $\expect \encase{Q_t^\sigma(\Delta \xv_t)} - \min_{\norm{\Delta\xv} \leq r}Q_t^\sigma(\Delta\xv) \leq (1 - \Theta)(Q_t^\sigma(\0) - \min_{\norm{\Delta\xv} \leq r}Q_t^\sigma(\Delta\xv))$ \label{eqn:accuracy}
			\State \emph{Update:} $\xv_{t+1} \leftarrow \xv_t + \Delta \xv_t$
		\EndFor
	\end{algorithmic}
\end{algorithm}
\subsection{Convergence analysis}
We need to quantify the approximation quality of the Hessian estimate $H_t$.
\begin{assumption}\label{asm:hessian-quality}
	We assume that there exists a constant $\eta$ such that for any $t \geq 0$, and $\zv_t = \xv_{t+1}$ as well as
  $\zv_t = \xv^\star$ the following holds
	\begin{align}
		\frac{1}{\eta} \norm{\zv_t - \xv_t}_{H_t} \leq \norm{\zv_t - \xv_t}_{\nabla^2 f(\xv_t)} \leq \eta \norm{\zv_t - \xv_t}_{H_t}\,. \tag{\ref{asm:hessian-quality}}
	\end{align}
\end{assumption}
Unfortunately the definition of $\eta$ is not necessarily affine invariant, but it does enable efficient approximations of the Hessian.
\begin{theorem}\label{thm:approx-newton}
	Given Assumptions \ref{asm:local-stable-hess-formal} and \ref{asm:hessian-quality}, for any iteration $T \geq 0$ of Algorithm \ref{alg:approx-newton} with $\sigma \geq \eta d(r)$,
	\[
		\expect\!\big[F(\xv_{T}) - F(\xv^\star)\big] \leq \encaser{1 - \frac{\Theta}{D\eta}\cdot\frac{r}{\sigma d(r)}}^T \big[ F(\xv_0) - F(\xv^\star) \big]\,,
	\]
  where $r$ is the trust region radius and $D$ is the diameter of the level set $\dom$.
\end{theorem}
\section{Optimality of results}\label{sec:optimality}
\paragraph{Linear vs. quadratic convergence.} When $f(x)$ is self-concordant, or strongly convex and smooth, Newton's method is known to converge quadratically when close enough to the optima~(\cite[Section 9.5]{boyd2004convex}). This was crucial in designing generic interior point algorithms~(\cite{wright1997primal}) and so one might ask if we can show similar local quadratic convergence for functions with stable Hessians? We give a simple counterexample for which Newton only achieves linear convergence. Consider $f(x) = x^{2k}$ for some large $k \geq 1$ and $c \geq 1$. The function has a minimum value of 0 achieved at 0, and $f'(x) = 2kx^{2k-1}$ and $f''(x) = 2k(2k -1)x^{2(k-1)}$. The Newton step on this function is $x_1 = x - \frac{x}{2k -1}$, with a decrease in the function of
\[
  \frac{f(x_1)}{f(x)} = \frac{(x - x/(2k -1))^{2k}}{x^{2k}} = \encaser{1 - \frac{1}{2k -1}}^{2k}\,.
\]
While $f$ is not globally stable, it is locally stable if at each step we restrict the trust region around~~$x$ to lie within $[x/2,x]$. Thus running Newton on $f$ with this varying trust region would also result in linear convergence, showing that our analysis can in general not be improved.

\paragraph{Superiority of trust region.} We saw in Section \ref{sec:trust} that trust-region Newton methods converge at a rate depending on the local stability of the Hessian. One might question if Newton's method equipped with line search could potentially have similar advantages. We provide a negative answer to this question. Consider the two dimensional function
\[
f(x,y) \defeq e^{-x} + x+ e^{-y} + y -2\,.
\]
The minimum of this function is 0 achieved at $(0,0)$. Let us pick a starting point $(x_0,y_0) = (k,-k)$. The Newton's update with step-size $\alpha$ can be computed to be
\[
	(x_1,y_1) = \encaser{k - \alpha (e^k - 1), -k + \alpha( 1 - e^{-k}) }\,.
\]
Suppose we perform an exact line search to pick the best $\alpha$. To simplify computations, we will look at the case where $k \rightarrow \infty$ i.e. when $k$ is large. In this setting, the predominant term in the objective is~$e^{-y_0}$. The optimal $\alpha$ in this case is approximately $\frac{k}{e^k} \xrightarrow{k \rightarrow \infty} 0$. This means that
\[
\lim\limits_{k \rightarrow \infty}\frac{f(x_1,y_1)}{f(k,-k)} = 1\,.
\]
Thus we cannot hope to obtain a global linear convergence for this case. However if we instead solve the quadratic problem defined by the Hessian as in Step \ref{eqn:subproblem}
\[
	Q_t^\sigma(x,y) \defeq (x - x_t)(1 - e^{x_t}) + \frac{\sigma e^{-x_t}}{2}(x - x_t)^2 + (y - y_t)(1 - e^{y_t}) + \frac{\sigma e^{-y_t}}{2}(y - y_t)^2
 \]
with the trust region $\abs{x - x_t}\leq 1$ and $\abs{y - y_t}\leq 1$, then the Hessian changes only by a factor of $e$. This means that the constant $d(r)$ as defined in Assumption~\ref{asm:local-stable-hess-formal} for $r= 1$ and using the $L_\infty$ norm is at most $e$. Thus we can use a constant step-size $1/\sigma = 1/e$, independent of $k$. As before if we look at what happens when $k \rightarrow \infty$, we get that
\[
	\lim\limits_{k \rightarrow \infty}\frac{f(x_1,y_1)}{f(k,-k)} = e^{-1/e}\,.
\]
This shows that trust region methods can be superior to line search methods, especially with a careful choice of the trust region.
\section{Conclusion}\label{sec:conclusion}

A predominant focus of past work on Newton methods
has been to show local quadratic convergence under very restrictive assumptions---both on the function class, as well as on the starting point. Such assumptions are almost never satisfied in practice, especially in machine learning applications. We believe the notion of stability recasts the analysis of Newton-type methods in a manner much more suitable to such applications. Using stability, we show strong global linear convergence rates under conditions in which first-order methods would only achieve sublinear rates---thereby providing a fresh perspective on the performance of a host of classical Newton's methods.

There are a number of follow-up questions which arise out of this work. Using the estimate sequence framework of \cite{nesterov_introductory_2014}, it is
 possible to accelerate the exact Newton's method. However it is unclear if such an acceleration could also be achieved for the trust-region methods, or for the approximate and proximal extensions.
Further, our theory indicates that the radius of the trust region is crucial for ensuring fast convergence. Although adaptive methods exist for picking the step-size (Appendix~\ref{sec:line-search}), designing and evaluating theoretically justified adaptive schemes for picking the trust-region radius would be a fruitful direction.
 Finally, the notion of stability is restricted to convex functions---generalizing insights here to the non-convex setting remains a challenging open problem.
\bibliographystyle{plain}
\bibliography{Newton}
\appendix
\part*{Appendix}
\section{Sufficient conditions for stability}\label{sec:sufficient}
Here follow the definitions of the various conditions on $f(\xv)$ discussed in Section \ref{sec:stability}. First some notation:
\begin{align*}
  \nabla f(\uv)[\dv] &= {\dotprod{\nabla f(\uv)}{\dv}} = \frac{d f(\uv + t\dv)}{dt}\Big\vert_{t=0}\,,\\
  \nabla^2 f(\uv)[\dv] &= {{\dv^\top \nabla^2 f(\uv)\dv}} = \frac{d^2 f(\uv + t\dv)}{dt^2}\Big\vert_{t=0},\ \text{and}\\
  \nabla^3 f(\uv)[\dv] &= \frac{d^3 f(\uv + t\dv)}{dt^3}\Big\vert_{t=0}\,.
\end{align*}
We will restate the definitions of these conditions using our new notation. For any $\uv ,\vv \in \dom$,
\begin{enumerate}
  \item $c$-stable Hessian: $\nabla^2 f(\vv)[\vv - \uv] \leq c \nabla^2 f(\uv)[\vv - \uv]$.
  \item $L$-Lipschitz gradients: $\nabla^2 f(\uv)[\vv - \uv] \leq L\norm{\vv - \uv}_2^2$.
  \item $\mu$-strongly convex: $\nabla^2 f(\uv)[\vv - \uv] \geq \mu\norm{\vv - \uv}_2^2$.
  \item $M$-Lipschitz Hessian: $\nabla^3 f(\uv)[\vv - \uv] \leq M\norm{\uv - \vv}_2^3$.
  \item $k$-self-concordant: $(\nabla^3 f(\uv)[\vv - \uv]) \leq 2k(\nabla^2 f(\uv)[\vv - \uv])^{3/2}$.
  \item $k$-quasi-self-/concordant: $(\nabla^3 f(\uv)[\vv - \uv]) \leq k \norm{\vv - \uv}_2(\nabla^2 f(\uv)[\vv - \uv])$.
\end{enumerate}
Also, recall the diameter of the level set $D = \max_{\uv,\vv \in \dom}\norm{\vv - \uv}_2$.
\paragraph{Proof of Theorem \ref{thm:sufficient}.}
Let us prove the Theorem case by case.
\begin{enumerate}
  \item $L$-Lipschitz gradient and $\mu$-strong convex $\,\Rightarrow\,$ $L/\mu$-stable Hessian.\\
    Using the definitions of the three terms,
    \[
      c \leq \frac{\nabla^2 f(\vv)[\vv - \uv]}{\nabla^2 f(\uv)[\vv - \uv]} \leq \frac{L \norm{\vv -\uv}_2^2}{\mu \norm{\vv - \uv}_2^2} = \frac{L}{\mu}\,.
    \]
  \item $M$-Lipschitz Hessian and $\mu$-strong convex $\,\Rightarrow\,$ $1 + \frac{MD}{\mu}$-stable Hessian.\\
    The definition of $M$-Lipschitz Hessian implies that
    \[
      \nabla^2 f(\vv)[\vv - \uv] - \nabla^2 f(\uv)[\vv - \uv] \leq M\norm{\vv - \uv}_2^3\,.
    \]
    Now combining this with the definition of stability and strong convexity,
    \[
      c \leq \frac{\nabla^2 f(\vv)[\vv - \uv]}{\nabla^2 f(\uv)[\vv - \uv]} \leq \frac{\nabla^2 f(\uv)[\vv - \uv] + M\norm{\vv - \uv}_2^3}{\nabla^2 f(\uv)[\vv - \uv]} \leq 1 + \frac{M\norm{\vv - \uv}_2^3}{\nabla^2 f(\uv)[\vv - \uv]} \leq 1 + \frac{MD}{\mu}\,.
    \]
  \item $k$-self-concordant and $L$-Lipschitz gradient $\,\Rightarrow\,$ $(1 + kDL)^2$-stable Hessian.\\
  We use the proof technique from \cite[Lemma 3.2]{GaoQuasiNewtonMethodsSuperlinear2016}. Define $\phi(t) = \dv^\top \nabla^2 f(\uv + t\dv)\dv$. Assuming $f$ is thrice differentiable, using the definition of self-concordance
  \[
    \abs{\phi'(t)} = \abs{\nabla^3 f(\uv)[\dv]} \leq 2k(\dv^\top \nabla^2 f(\uv + t\dv)\dv)^{3/2} = 2k\phi(t)^{3/2}\,.
  \]
  This means that the definition of self-concordance implies that
  \[
    \abs{\frac{d}{dt}\phi(t)^{-1/2}} = \frac{1}{2}\abs{\phi(t)^{-3/2}\phi'(t)} \leq k\,.
  \]
  Since $\phi(t)^{-1/2}$ is $k$-Lipschitz, this means
  \[
    \phi(0)^{-1/2} \leq \phi(1)^{-1/2} + k\,.
  \]
  Now setting $\dv = \vv - \uv$ in the definition of $\phi(t)$ and multiplying the above equation by $\phi(1)$ we get
  \[
   \frac{\nabla^2 f(\vv)[\vv -\uv]^{1/2}}{\nabla^2 f(\uv)[\vv -\uv]^{1/2}} = \frac{\phi_t(1)^{1/2}}{\phi_t(0)^{1/2}} \leq 1 + \frac{2k\phi_t(1)^{1/2}}{2} = 1 + k\nabla^2 f(\vv)[\vv -\uv]\,.
  \]
  Using the definition of Lipschitz gradient, and the bound on the diameter of $\dom$, we get that for all $\uv,\vv \in \dom$,
  \[
  c \leq \frac{\nabla^2 f(\vv)[\vv -\uv]}{\nabla^2 f(\uv)[\vv -\uv]} \leq (1 + kLD)^2\,.
  \]
  \item $k$-quasi-self-concordant $\,\Rightarrow\,$ $\exp(kD)$-stable Hessian.\\
  This statement is directly taken from \cite[Proposition 1]{bach_self-concordant_2010}. Define as before $\phi(t) = \dv^\top \nabla^2 f(\uv + t\dv)\dv$. The definition of $k$-quasi-self-concordance implies that
  \[
    \phi'(t) = {\nabla^3 f(\uv)[\dv]} \leq k\norm{\dv}_2(\nabla^2 f(\uv)[\dv]) = k\norm{\dv}_2\phi(t)\,.
  \]
  If we consider the function $\log(\phi(t))$, the above equation shows that
  \[
    \frac{d}{dt}\log(\phi(t)) \leq k\norm{\dv}_2\,,
  \]
  which in turn means
  \[
    \phi(1) \leq \exp(k\norm{\dv})\phi(0)\,.
  \]
  Again setting $\dv = \vv - \uv$ in the definition of $\phi(t)$ gives us that
  \[
    c \leq \frac{\nabla^2 f(\vv)[\vv -\uv]}{\nabla^2 f(\uv)[\vv -\uv]} \leq \exp(kD)\,.
  \]
\end{enumerate}

\section{Line search strategies}\label{sec:line-search}
All algorithms we have discussed in this paper assume that the value of $\sigma$ is set correctly. This assumption can easily be relaxed by using line search strategies. There has been a significant amount of work different line-search strategies and we will not attempt to provide a complete survey. Instead we point to (\cite{ConnTrustregionMethods2000}). Among those methods, the backtracking strategy employed in making the cubic regularization techniques adaptive by (\cite{CartisAdaptivecubicregularisation2011a}) is especially suited to our strategy.
%
\begin{algorithm}
	\caption{Back tracking strategy}
		\label{alg:backtracking}
	\begin{algorithmic}[1]
		\State \textbf{Input:} $\xv_0$, $\sigma_0 = 1$, $\zeta_1 > \zeta_2 \in [0,1)$, and $\eta_2 \geq \eta_1 >1$
		\For{$t = \{0,\dots\}$}
		    \State \emph{Define quadratic subproblem:}
		    \[Q_t^{\sigma_t}(\Delta \xv) \defeq \dotprod{\nabla f(\xv_t)}{\Delta \xv} + \frac{\sigma_t}{2}\norm{\Delta \xv}_{H_t}^2 + g(\xv_t + \Delta \xv) - g(\xv_t)\]
			\State \emph{Compute update:} Let $\Delta \xv_t^{\sigma_t}$ be the update based on $Q_t^{\sigma_t}(\Delta \xv)$
      \State \emph{Check progress:} Compute $F(\xv_t + \Delta \xv_t^{\sigma_t})$ and $\rho_t \defeq \frac{F(\xv_t + \Delta \xv_t^{\sigma_t}) - F(\xv_t)}{Q_t^{\sigma_t}(\Delta \xv_t^{\sigma_t}) - Q_t^{\sigma_t}(0)}$
      \State \begin{equation*}
        \xv_{t+1}=
        \begin{cases}
          \xv_{t}, & \text{if}\ \rho_t <
          \zeta_2 \quad(\text{{\bad} iteration})\\
          \xv_{t} + \Delta \xv_t^{\sigma_t}, & \text{otherwise}
        \end{cases}
      \end{equation*}
      \State \begin{equation*}
        \sigma_{t+1}=
        \begin{cases}
          \sigma_t/\eta_1, & \text{if}\ \rho_t >
            \zeta_1 \quad (\text{{\verygood} iteration})\\
          \sigma_t, & \text{if}\ \rho_t \in [\zeta_2, \zeta_1] \quad (\text{{\good} iteration})\\
          \eta_2 \sigma_t, & \text{if}\ \rho_t
          < \zeta_2 \quad (\text{{\bad} iteration})
        \end{cases}
      \end{equation*}
		\EndFor
	\end{algorithmic}
\end{algorithm}

It is easy to adapt the theoretical guarantees and techniques used in (\cite{CartisAdaptivecubicregularisation2011a,scheinberg2016practical}) for analysis of this backtracking strategy to our setting. This way we are able to remove both the necessity of knowing $\sigma$ as well as make it an adaptive method. The details are summarized in Algorithm \ref{alg:backtracking}.

Throughout Algorithm \ref{alg:backtracking}, we always assumed that the only unknown parameter is $\sigma$. However when we are running trust region algorithms, we would perhaps like to adapt both the trust region radius $r$ as well as $\sigma$. While it is possible to design such an adaptive trust region strategy using insights from on our proof, we leave the analysis and evaluation of such strategies for future work.

\paragraph{Necessity of step-size.}
Theorems \ref{thm:exact-newton},  \ref{thm:trust-newton} and \ref{thm:approx-newton} show that choosing the appropriate $\sigma$ ensures global linear convergence. In the case where $g(\xv) = 0$, this corresponds to using a step-size of ${1}/{\sigma}$. Here we show that this is not simply an artifact of the analysis---the use of $\sigma\ne1$ is actually necessary to ensure global convergence. Consider the univariate function
\[
	f(x) \defeq e^{-x} + x - 1\,.
\]
This function is convex with gradient $f'(x) = -e^{-x} + 1$, second derivate $f''(x) = e^{-x} \geq 0$ and minimum value 0 achieved at $x = 0$. It satisfies our condition (Assumption \ref{asm:stable-hess-formal}) of stable Hessian with $c = e^{D}$ where $D$ is the diameter of the level set. Suppose we start at $x_0 = k$ for $ k \geq 1$. Applying a Newton update with step-size $\alpha$ gives $x_1 = k - \alpha \frac{-e^{-k} + 1}{e^{-k}} = k+ \alpha (1 - e^k)$. Let us assume the step-size $\alpha =1$, and $k \rightarrow \infty$ to simplify computations. When $\abs{x}$ is large, the predominant term of $f(x)$ is $x$ if $x \geq 0$ and $e^{-x}$ if $x < 0$. In the setting where $k \rightarrow \infty$, $x_1 = k+ 1 - e^k \approx - e^k$ and $f(x_1) \approx e^{e^k}$---we have veered too far to the left. Instead,
using a step-size $\alpha = 1/c = 1/e^k$ would ensure a descent step. In fact this example also showcases the advantage of adaptive step-sizes. Using a fixed step-size of either $\alpha = 1$ or even $\alpha = 1/e^k$ would require exponential (in $k$) number of iterations to converge. Instead, using an adaptive step size of $1/e^x$, where $x$ is the current position, would give convergence in polynomial steps.
\section{Additional proofs}\label{sec:proofs}

\subsection{Proof of affine invariance of stability (Lemma \ref{lem:affine-inv})}
	Suppose we had a transformed function $h(\uv) = f(A\uv)$ for an invertible matrix $A$. Its Hessian would be $\nabla^h(\uv) = A^\top \nabla^2 f(A\uv) A$, using the chain rule. Let $A^{-1}\dom$
  denote the transformed domain of $h(\uv)$ as defined in \eqref{eqn:domain} so that $\uv \in A^{-1}\dom$ if $f(A\uv) \leq f(A\uv_0)$.  The definition of $c$ would be
	\begin{align*}
		c &= \max_{\uv, \vv \in A^{-1}\dom} \frac{\norm{\uv - \vv}_{\nabla^2 h(\vv)}^2}{\norm{\uv - \vv}_{\nabla^2 h(\uv)}^2}\\
		&= \max_{\uv, \vv \in A^{-1}\dom} \frac{\norm{A(\uv - \vv)}_{\nabla^2 f(A\vv)}^2}{\norm{A(\uv - \vv)}_{\nabla^2 f(A\uv)}^2}\\
		&= \max_{\xv, \yv \in \dom} \frac{\norm{\xv - \yv}_{\nabla^2 f(\yv)}^2}{\norm{\xv - \yv}_{\nabla^2 f(\xv)}^2}\,.
	\end{align*}
\qedhere

\subsection{Proof of lower and upper bounds (Lemma \ref{lem:upperlower-simple})}\label{subsec:simple-bounds-proof}
The proof of the Lemma follows from the second-order Taylor expansion of $f(\yv)$ around $\xv$. Taylor's theorem gives us that for any $\xv, \yv$ there exists a $\gamma \in [0,1]$ such that for $\zv = (1- \gamma)\xv+ \gamma \yv$,
\begin{equation}
	f(\yv) = f(\xv) + \dotprod{\nabla f(\xv)}{\yv - \xv} + \frac{1}{2}\norm{\yv - \xv}_{\nabla^2 f(\zv)}^2 \,.
\end{equation}
Since $\dom$ is convex, $\zv \in \dom$ and by substituting $\uv = \xv$, $\vv = \zv$ in Assumption \ref{asm:stable-hess-formal},
\[
	\norm{\zv - \xv}_{\nabla^2 f(\zv)}^2 \leq c \norm{\zv - \xv}_{\nabla^2 f(\xv)}^2\,.
\]
Substituting $\zv - \xv = \gamma (\yv - \xv)$, we have
\[
\gamma^2\norm{\yv - \xv}_{\nabla^2 f(\zv)}^2 \leq c\gamma^2 \norm{\yv - \xv}_{\nabla^2 f(\xv)}^2\,.
\]
This proves \eqref{eqn:upper-simple}. On the other hand, by substituting $\uv = \zv$, and $\vv = \xv$ in Assumption \ref{asm:stable-hess-formal}, we get a lower bound
\[
	c \geq \frac{\norm{\zv - \xv}_{\nabla^2 f(\xv)}^2}{\norm{\zv - \xv}_{\nabla^2 f(\zv)}^2} \,.
\]
Again by substituting $\zv - \xv = \gamma (\yv - \xv)$, we can finish the proof as
\[
	c {\norm{\yv - \xv}_{\nabla f(\zv)}^2} \geq {\norm{\yv - \xv}_{\nabla^2 f(\xv)}^2} \,. \qedhere
\]
\qed

\subsection{Proof of descent (Lemma \ref{lem:descent})}
\label{sec:proofdescent}
  For some $t \geq 0$, let us assume that $\xv_t \in \dom$. The base case, $\xv_0 \in \dom$ is trivially true. If $\nabla f(\xv_t) = \0$, we are already at an optimum and $\xv_{t+1} = \xv_t$, proving our Lemma. Otherwise we proceed as below.

  We know that $-\nabla f(\xv_t)$ is a descent direction \cite[Section~9.2]{boyd2004convex}. This means there exists a small enough $\gamma >0$ such that for $\yv = \xv_t - \gamma \nabla f(\xv_t)$,
  $f(\yv) \leq f(\xv_t) \leq f(\xv_0)$ meaning $\yv \in \dom$. Applying Assumption \ref{asm:stable-hess-formal} with $\uv = \xv_t$ and $\vv = \yv$, we get that $\norm{\yv - \xv_t}_{\nabla^2 f(\xv_t)} = \gamma^2\norm{\nabla f(\xv_t)}_{\nabla^2 f(\xv_t)} > 0$. In particular this implies that $\nabla f(\xv_t)$ is in the range of $\nabla^2 f(\xv_t)$ and so the update $[\nabla^2 f(\xv_t)]^\dagger \nabla f(\xv_t)$ is well-defined.

  Now we are left with the task of proving $f(\xv_{t+1}) \leq f(\xv_t)$. Note that $\norm{\nabla f(\xv_t)}_{\nabla^2 f(\xv_t)} > 0$ also implies $\dotprod{[\nabla^2 f(\xv_t)]^\dagger\nabla f(\xv_t)}{\nabla f(\xv_t)} >0$. This is a sufficient condition to ensure that the Newton's step is a descent direction \cite[Section~9.2]{boyd2004convex}. This means there exists $\gamma$, $0< \gamma \leq 1/c$ such that for $\yv_{\gamma} \defeq \xv_t - \gamma[\nabla^2 f(\xv_t)]^\dagger\nabla f(\xv_t)$, we have $f(\yv_{\gamma}) \leq f(\xv_t) \leq f(\xv_0)$. Hence $\yv_{\gamma} \in \dom$.
  Let us define the auxiliary function $h(\alpha) = f(\yv_{\alpha})$ for $\alpha > 0$. The function $h(\alpha)$ is continuous in $\alpha$ since $f(\xv)$ is a continuous function and $\yv_{\alpha}$ is a continuous map. Moreover we have that
  \[
    \lim_{\alpha \rightarrow 0} h(\alpha) = f(\xv_t)\,.
  \]
  We know that $h(0) = f(\xv_t)$ and for some $0 < \gamma \leq 1/c$, $h(\gamma) \leq f(\xv_t)$. Suppose that $h(1/c) > f(\xv_t)$; otherwise we are done. Since $h(\alpha)$ is a continuous function, by the the intermediate value theorem, there must exist $\beta \in [\gamma, 1/c)$ such that $h(\beta) = f(\xv_t)$. This also implies that $\yv_\beta \in \dom$ and so the upper bound \eqref{eqn:upper-simple} in Lemma \ref{lem:upperlower-simple} holds. In other words,
  \begin{align*}
    f(\yv_\beta) - f(\xv_t) &\leq \dotprod{\nabla f(\xv_t)}{\yv_\beta - \xv_t} + \frac{c}{2}\norm{\yv_\beta - \xv_t}_{\nabla^2 f(\xv_t)}^2\\
    &= (\beta^2 c/2 - \beta)\norm{\nabla^2 f(\xv_t)^\dagger\nabla f(\xv_t)}^2_{\nabla^2 f(\xv_t)}\\
    &\leq -\beta/2\norm{\nabla^2 f(\xv_t)^\dagger \nabla f(\xv_t)}^2_{\nabla^2 f(\xv_t)}\\
    &< 0\,.
  \end{align*}
  In the final two inequalities, we used that $1/c \geq \beta > 0$. This is clearly a contradiction since we had picked $\beta$ such that $f(\yv_\beta) = f(\xv_t)$. Thus $h(1/c) = f(\yv_{1/c}) = f(\xv_{t+1}) \leq f(\xv_t) \leq f(\xv_0)$ and so the algorithm is a descent method.
\qed

\subsection{Proof of approximate proximal Newton method (Theorem \ref{thm:approx-newton})}
\label{sec:proof-of-trustregion}
Because we assume a multiplicative error bound on our Hessian approximation we can combine it with the stability of the Hessian. As before, we define $\gamma = \frac{r}{D}$ where $D$ is the diameter of $\dom$. Also, for each iteration $t$, define $\xv^\star_\gamma = (1- \gamma)\xv_t + \gamma \xv^\star$.
\begin{lemma}\label{lem:newconstants}
	Under Assumptions \ref{asm:local-stable-hess-formal} and \ref{asm:hessian-quality}, for any $\alpha \in [0,1]$, $\Delta \xv_t = \xv_{t+1} - \xv_t$, $\gamma = r/D$, and $\xv^\star_\gamma = (1-\gamma)\xv_t + \gamma\xv^\star$, the following two conditions hold:
  	\begin{align}
  		\text{Upper bound:}&\quad\quad \norm{\Delta \xv_t}_{\nabla^2 f(\xv_t + \alpha \Delta \xv_t)}^2 \leq {d(r) \eta}\norm{\Delta \xv_t}_{H_t}^2\,, \text{ and} \label{eqn:hessupper}\\
  		\text{Lower bound:}&\quad  \norm{\xv^\star - \xv_t}_{\nabla^2 f((1 - \alpha)\xv_t + \alpha \xv^\star)}^2 \geq \frac{1}{d(r) \eta}\norm{\xv^\star_\gamma - \xv_t}_{H_t}^2\,. \label{eqn:hesslower}
  	\end{align}
\end{lemma}

There are two main components to the proof of Theorem \ref{thm:approx-newton}. The first is a generalization of Lemma \ref{lem:upperlower-simple} which shows that Lemma \ref{lem:newconstants} implies upper and lower bounds on the function value.
\begin{lemma}\label{lem:upper-lower-bound}
	Assuming that conditions of Lemma \ref{lem:newconstants} hold, the function value at $\xv_{t+1}$ can be bounded from above and below as follows
	\begin{align}
		\text{Upper bound:}\quad F(\xv_{t+1}) - F(\xv_t) &\leq Q_t^{d(r)\eta}(\Delta \xv_t)\,, \text{ and} \label{eqn:sigmaupper}\\
		\text{Lower bound:}\quad  \ \ F(\xv_t) - F(\xv^\star_\gamma) &\geq Q_t^{1/(d(r) \eta)}(\xv^\star_\gamma - \xv_t)\,. \label{eqn:sigmalower}
	\end{align}
\end{lemma}
\begin{proof}
	The proof follows exactly along the lines of the proof of Lemma \ref{lem:upperlower-simple}. Using the second-order Taylor expansion of $f(\xv)$, for any $\xv$ and $\yv$, there exists a $\gamma \in [0,1]$ such that for $\zv = (1 - \gamma)\xv + \gamma \yv$,
	\[
		f(\yv) = f(\xv) + \dotprod{\nabla f(\xv)}{\yv - \xv} + \frac{1}{2}\norm{\yv - \xv}_{\nabla^2 f(\zv)}^2\,.
	\]
	To obtain \eqref{eqn:sigmaupper}, we use the above equation with $\yv = \xv_{t+1}$, $\xv = \xv_t$ and the upper bound \eqref{eqn:hessupper}. Similarly, for \eqref{eqn:sigmalower} we can use the Taylor expansion with $\yv = \xv^\star_\gamma$, $\xv = \xv_t$ and the lower bound \eqref{eqn:hesslower}.
\end{proof}
We also borrow a very useful technical Lemma from \cite[Lemmata 2 \& 9]{karimireddy_adaptive_2018-1} which allows us to relate the minimum values of the two quadratic subproblems.
\begin{lemma}\label{lem:changing-sigma}
	For any convex domain $\dom$ and constants $\alpha\cdot \beta \geq 1$,
	$$
	\min_{\Delta \xv \in \dom} Q^{\alpha}(\Delta \xv)  \leq \frac{1}{\alpha \beta} \min_{\Delta \xv  \in \dom} Q^{1/\beta}(\Delta \xv)\,.
	$$
\end{lemma}
We are now ready to prove Theorem \ref{thm:approx-newton}.
\paragraph{Proof.}
For the sake of convenience, we will define
\[
  (Q_t^\sigma)^\star = \min_{\norm{\Delta \xv} \leq r}Q_t^\sigma(\Delta \xv)\,.
\]
The assumption that we solved our subproblem to $\Theta$ accuracy as in Step \ref{eqn:accuracy} means that
\begin{align*}
	\expect_t[Q_t^\sigma(\Delta \xv_t)] &\leq (Q_t^\sigma)^\star + (1 - \Theta)( Q_t^\sigma(\0) -  (Q_t^\sigma)^\star)\\
	&= \Theta  (Q_t^\sigma)^\star + (1 - \Theta) Q_t^\sigma(\0)\\
	&= \Theta  (Q_t^\sigma)^\star\,.
\end{align*}
The last equality follows because $ Q_t^\sigma(\0) = 0$ from the definition. Using the upper bound \eqref{eqn:sigmaupper} of Lemma \ref{lem:upper-lower-bound}, we have that
\begin{align*}
	\expect_t[F(\xv_{t+1}) - F(\xv_t)] &\leq \expect_t[Q_t^\sigma(\Delta \xv_t)] \\
	&\leq \Theta  \min_{\norm{\Delta \xv} \leq r}Q_t^\sigma(\Delta x)\\
	&\leq \frac{\Theta}{d(r)\eta \sigma}  \min_{\norm{\Delta \xv} \leq r}Q_t^{1/(d(r)\eta)}(\Delta x)
\end{align*}
In the last inequality, we used Lemma \ref{lem:changing-sigma} since $\sigma \geq d(r)\eta \geq 1$. Now using the lower bound \eqref{eqn:sigmalower} of Lemma \ref{lem:upper-lower-bound},
\begin{align*}
	\expect_t[F(\xv_{t+1})] - F(\xv_t) &\leq \frac{\Theta}{c\eta \sigma} \min_{\norm{\Delta \xv} \leq r}Q_t^{1/(d(r)\eta)}(\Delta x)\\
	&\leq\frac{\Theta}{d(r)\eta \sigma}Q_t^{1/(d(r)\eta)}(\xv^\star_\gamma - \xv_t)\\
	&\leq\frac{\Theta}{d(r)\eta \sigma} [F(\xv^\star_\gamma) - F(\xv_t)]\,.
\end{align*}
Using the convexity of $F(\xv)$, we have that
\begin{align*}
  \expect_t[F(\xv_{t+1})] - F(\xv_t) &\leq \frac{\Theta}{d(r)\eta \sigma} [F(\xv^\star_\gamma) - F(\xv_t)]\\
  &\leq \frac{\Theta}{d(r)\eta \sigma} [F((1-\gamma)\xv_t + \gamma \xv^\star) - F(\xv_t)]\\
  &\leq \frac{\Theta}{d(r)\eta \sigma}[(1-\gamma)F(\xv_t) + \gamma F(\xv^\star) - F(\xv_t)]\\
  &= \frac{\gamma \Theta}{d(r)\eta \sigma}[F(\xv^\star) - F(\xv_t)]\,.
\end{align*}
Adding and subtracting $F(\xv^\star)$, and rearranging the terms finishes the proof.
\qed

\section{Affine invariant trust region algorithm}\label{sec:appendix-affine-trust}
As we discussed in Section \ref{subsec:trust-rate-improved}, traditional trust region algorithms have two drawbacks: i) they require minimizing the subproblem $Q_t^\sigma(\Delta \xv)$ from Step \ref{eqn:subproblem} with the additional constraint that $\norm{\Delta \xv} \leq r$, and ii) the algorithm is no more affine-invariant. In the case where the function $F(\xv) = f(\xv) + g(\xv)$ is such that $g(\xv)$ is the indicator function of a bounded domain $\dom$, it is possible to overcome these limitations. So we are interested in the following problem
\[
  \min_{\xv \in \dom} f(\xv)\,.
\]
At each iteration for some $\gamma \in [0,1]$ we solve the problem
\begin{equation}\label{eqn:trust-subproblem}
	P_t^{\gamma,\sigma}(\yv) \defeq \dotprod{ \nabla f(\xv_t)}{\yv - \xv_t} + {\frac{\gamma \sigma}{2}\norm{\yv - \xv_t}_{\nabla^2 f(\xv_t)}^2}\,.
\end{equation}
We minimize $P_t^{\gamma,\sigma}(\yv)$ restricted to the domain $\dom$ to a multiplicative accuracy $\Theta$. Let $\sv_t \in \dom$ be such that
\begin{equation}\label{eqn:trust-subproblem-accuracy}
		\expect[P_t^{\gamma,\sigma}(\sv_t)] - 	\min_{\yv \in \dom}P_t^{\gamma,\sigma}(\yv) \leq (1 - \Theta)\encaser{P_t^{\gamma,\sigma}(\xv_t) - 	\min_{\yv \in \dom}P_t^{\gamma,\sigma}(\yv)}\,.
\end{equation}
Using this $\Theta$ approximate solution, we perform the following update
\[
	\xv_{t+1} = (1 - \gamma)\xv_t + \gamma \sv_t\,.
\]
We have effectively replaced the restriction that $\norm{\xv_{t+1} - \xv_t} \leq r$ by using the step-size $\gamma$. Suppose we shrink the domain $\dom$ by a factor $\gamma$ denoted by $\dom^\gamma(\xv) \defeq \{(1- \gamma)\xv + \gamma \vv \,|\,\vv \in \dom \}$. Then the updates we perform effectively uses the trust region $\dom^\gamma(\xv_t)$ to minimize $Q_t^\sigma(\Delta \xv)$. The details are summarized in Algorithm \ref{alg:invariant-trust-newton}.
\begin{algorithm}
	\caption{Affine-invariant trust-region Newton Descent}
	\label{alg:invariant-trust-newton}
	\begin{algorithmic}[1]
		\State \textbf{Input:} $\xv_0$,$\gamma \in (0,1]$, and $\sigma \geq \eta c(\gamma)$, .
		\For{$t = \{0,\dots\}$}
		\State \emph{Define quadratic subproblem:}
		\State \hspace{0.3cm}
		$P_t^{\gamma,\sigma}(\yv) \defeq \dotprod{ \nabla f(\xv_t)}{\yv - \xv_t} + {\frac{\gamma \sigma}{2}\norm{\yv - \xv_t}_{\nabla^2 f(\xv_t)}^2}$
		\State \emph{ Approximately minimize subproblem:} Find $\sv_t \in \dom$ such that
		\State \hspace{0.3cm} $\expect[P_t^{\gamma,\sigma}(\sv_t)] - 	\min_{\yv \in \dom}P_t^{\gamma,\sigma}(\yv) \leq (1 - \Theta)\encaser{P_t^{\gamma,\sigma}(\xv_t) - 	\min_{\yv \in \dom}P_t^{\gamma,\sigma}(\yv)}$
		\State \emph{Update:} $\xv_{t+1} \leftarrow (1 - \gamma)\xv_t + \gamma \sv_t$
		\EndFor
	\end{algorithmic}
\end{algorithm}
\subsection{Convergence analysis}
We can generalize the definition of stability of the Hessian in Assumptions \ref{asm:stable-hess-formal} and \ref{asm:local-stable-hess-formal} to give an affine invariant measure suitable for use in trust region methods.
\begin{assumption}[$c(\gamma)$-locally stable]\label{asm:invariant-trust-stability}
		For any  $\uv, \vv \in \dom$, $\uv \neq \vv$ and $\gamma \in [0,1]$, define $\wv = (1 - \gamma)\uv + \gamma \vv$. Then $\norm{\wv - \uv}_{\nabla^2 f(\uv)} > 0$ and there exists a non-decreasing function $c: [0,\infty) \rightarrow [1, \infty)$ such that $c(0) = 1$ and for $\gamma > 0$,
		\[
			c(\gamma) \defeq \max_{\uv, \vv \in \dom} \frac{\norm{\wv - \uv}_{\nabla^2 f(\wv)}^2}{\norm{\wv - \uv}_{\nabla^2 f(\uv)}^2} \,.
		\]
\end{assumption}
Note that $c(1) = c$ where $c$ is the global stability from Assumption \ref{asm:stable-hess-formal}.
\begin{theorem}\label{thm:invariant-trust-newton}
	Given Assumption \ref{asm:hessian-quality} and \ref{asm:invariant-trust-stability}, for any iteration $t \geq 0$ of Algorithm \ref{alg:invariant-trust-newton} with $\sigma \geq \eta c(\gamma)$ and any $\gamma \in (0,1]$
	\[
	\expect_t[F(\xv_{t+1}) - F(\xv^\star)] \leq \encaser{1 - \frac{\Theta}{\eta }\cdot\frac{\gamma}{\sigma c(\gamma)}}(F(\xv_{t}) - F(\xv^\star))\,.
	\]
\end{theorem}
\begin{proof}
  First note that applying Assumption \ref{asm:invariant-trust-stability} with $\uv = \xv_t$, $\vv=\sv_t$ and $\wv = \xv_{t+1}$ gives that for any $\alpha \in [0,1]$,
  \[
    \norm{\xv_{t+1} - \xv_t}^2_{\nabla^2 f(\xv_{t})} \leq c(\gamma) \norm{\xv_{t+1} - \xv_t}^2_{\nabla^2 f((1 - \alpha)\xv_{t} + \alpha \xv_t)}\,.
  \]
  We use this to replace \eqref{eqn:sigmaupper} in Lemma \ref{lem:upper-lower-bound} and obtain the following upper bound on the function value
  \[
    f(\xv_{t+1}) - f(\xv_t) \leq Q^{\eta c(\gamma)}_t(\xv_{t+1} - \xv_t)\,.
  \]
  Similarly for $\xv^\star_\gamma = (1-\gamma)\xv_t + \gamma \xv^\star$, we replace \eqref{eqn:sigmalower} with
  \[
    f(\xv^\star_\gamma) - f(\xv_t) \geq Q^{1/(\eta c(\gamma))}_t(\xv^\star_\gamma - \xv_t)\,.
  \]
  \begin{align*}
    f(\xv_{t+1}) - f(\xv_t) &\leq Q^{\eta \sigma}_t(\xv_{t+1} - \xv_t) \\
    &= \dotprod{\nabla f(\xv_t)}{\xv_{t+1} - \xv_t} + \frac{\sigma}{2}\norm{\xv_{t+1} - \xv_t}_{\nabla^2 f(\xv_t)}^2 \\
    &= \gamma\dotprod{\nabla f(\xv_t)}{\sv_{t} - \xv_t} + \frac{\gamma^2 \sigma}{2}\norm{\sv_{t} - \xv_t}_{\nabla^2 f(\xv_t)}^2 \\
    &= \gamma P_t^{\gamma,\sigma}(\sv_t)\,.
  \end{align*}
Now we will use that $\sv_t$ was approximated to $\Theta$ accuracy to get that
\begin{align*}
    f(\xv_{t+1}) - f(\xv_t) &\leq \gamma P_t^{\gamma,\sigma}(\sv_t)\\
    &\leq \gamma\Theta \min_{\yv \in \dom}P_t^{\gamma,\sigma}(\yv)\\
    &= \gamma\Theta \min_{\yv \in \dom}Q_t^{\gamma \sigma}(\yv)\,.
\end{align*}
Now let us use Lemma \ref{lem:changing-sigma} to go from the upper bound to the lower bound. For the lemma to be applicable, it is crucial that $\sigma \eta c(\gamma) \geq 1$.
\begin{align*}
  f(\xv_{t+1}) - f(\xv_t) &\leq \gamma\Theta \min_{\yv \in \dom}Q_t^{\gamma \sigma}(\yv)\\
  &\leq \gamma\Theta \cdot \frac{1}{\sigma \eta c(\gamma)}\min_{\yv \in \dom}Q_t^{\gamma/(c(\gamma)\eta)}(\yv) \\
  &= \frac{\Theta}{\sigma \eta c(\gamma)} \gamma Q_t^{\gamma/(c(\gamma)\eta)}(\xv^\star)\\
  &= \frac{\Theta}{\sigma \eta c(\gamma)} \encase{\gamma\dotprod{\nabla f(\xv_t)}{\xv^\star - \xv_t} + \frac{\gamma^2}{2\eta c(\gamma)}\norm{\xv^\star - \xv_t}^2_{\nabla^2 f(\xv_t)}}\\
  &= \frac{\Theta}{\sigma \eta c(\gamma)} \encase{\dotprod{\nabla f(\xv_t)}{\xv^\star_\gamma - \xv_t} + \frac{1}{2\eta c(\gamma)}\norm{\xv^\star_\gamma - \xv_t}^2_{\nabla^2 f(\xv_t)}}\\
  &= \frac{\Theta}{\sigma \eta c(\gamma)}Q_t^{1/(c(\gamma)\eta)}(\xv^\star_\gamma)\\
  &= \frac{\Theta}{\sigma \eta c(\gamma)}[f(\xv^\star_\gamma) - f(\xv_t)]\,.
\end{align*}
Now we will use the convexity of $f(\xv)$ and the definition of $\xv^\star_\gamma$ to note that
\begin{align*}
  f(\xv_{t+1}) - f(\xv_t) &\leq \frac{\Theta}{\sigma \eta c(\gamma)}[f(\xv^\star_\gamma) - f(\xv_t)]\\
  &\leq \frac{\Theta}{\sigma \eta c(\gamma)}[f((1- \gamma)\xv_t + \gamma \xv^star) - f(\xv_t)]\\
  &\leq \frac{\Theta}{\sigma \eta c(\gamma)}[(1- \gamma)f(\xv_t) + \gamma f(\xv^star) - f(\xv_t)]\\
  &= \frac{\Theta\gamma}{\sigma \eta c(\gamma)}[f(\xv^\star) - f(\xv_t)]
\end{align*}
Adding and subtracting $f(\xv^\star)$ on the left hand side, rearranging the terms and iterating over $t$ finishes the proof.
\end{proof}
\end{document}